\documentclass[a4paper,UKenglish,cleveref, autoref, thm-restate]{lipics-v2021}

\usepackage{graphics}

\newcommand{\eclause}{{\unitlength 2.3mm \,\framebox(1,1){}\,}}
\newcommand{\wcla}[2]{_{(#1)}\,#2}
\newcommand{\xor}{\oplus}
\newcommand{\ignorar}[1]{}
\newcommand{\opt}{Opt}
\newcommand{\cost}{Cost}
\newcommand{\weight}{Weight}

\title{Reducing SAT to Max2XOR}
\author{Carlos Ansótegui}{Logic \& Optimization Group (LOG), University of Lleida, Spain}{carlos@diei.udl.cat}{https://orcid.org/0000-0001-7727-2766}{This work is supported by Grant PID2019-109137GB-C21 funded by
MCIN/AEI/10.13039/501100011033.}
\author{Jordi Levy}{IIIA-CSIC, Spain}{levy@iiia.csic.es}{https://orcid.org/0000-0001-5883-5746}{This work is supported by Grant PID2019-109137GB-C21 funded by \hfill\mbox{}
MCIN/AEI/10.13039/501100011033.}
\authorrunning{C. Ansótegui and J. Levy} %TODO mandatory. First: Use \authorrunning{C. Ansótegui \and J. Levy} %TODO mandatory. First: Use abbreviated first/middle names. Second (only in severe cases): Use first author plus 'et al.'
\Copyright{Carlos Ansótegui and Jordi Levy} %TODO mandatory, please use full first names. LIPIcs license is "CC-BY";
 % TODO

\date{February 2022}

\ArticleNo{arXiv}
\DOIPrefix{}

\nolinenumbers

\begin{document}

\maketitle

\begin{abstract}
Representing some problems with XOR clauses (parity constraints) can allow to apply more efficient reasoning techniques. In this paper, we present a gadget for translating SAT clauses into Max2XOR constraints, i.e., XOR clauses of at most 2 variables equal to zero or to one.  Additionally, we present new resolution rules for the Max2XOR problem which asks for which is the maximum number of constraints that can be satisfied from a set of 2XOR equations. 
\end{abstract}

\section{Introduction}

Exclusive OR (XOR), here written $\xor$, may be an alternative to the use of traditional OR to represent propositional formulas. By writing clauses $x_1\xor \cdots \xor x_k$ as constraints $x_1\xor \cdots \xor x_k=1$, where $1$ means true and $0$ false, we can avoid the use of negation, because $\neg x\xor C=k$ is equivalent to $x\xor C=k\xor 1$, where by abuse of notation, $\xor$ on the right-hand side denotes the addition modulo $2$. The equivalent to the resolution rule for XOR constraints, called XOR resolution rule, is
\[
\begin{array}{c}
X\xor A=k_1\\
X\xor B=k_2\\
\hline
A\xor B = k_1\xor k_2
\end{array}
\]
where $X$, $A$ and $B$ are clauses. In the particular case of $A=B=\emptyset$ and $k_1\neq k_2$, this rule concludes a contradiction, that we represent as $\eclause$. The proof system containing only this rule allows us to produce polynomial refutations for any unsatisfiable  set of XOR constraint. Therefore, unless $P=NP$, we cannot express any propositional formula as an equivalent conjunction of XOR constraints. It is also well-known that a set of XOR-constraints can be solved in polynomial time using Gaussian elimination. In practice, many approaches combining CNF and XOR reasoning have been presented \cite{Li00,Li00prima,BaumgartnerM00,Li03,HeuleM04,HeuleDZM04,Chen09,SoosNC09,LaitinenJN12sat12,Soos10,LaitinenJN11,LaitinenJN12ictai12,SoosM19}.

In this paper, we focus on the problem of reducing the satisfiability of a CNF formula (SAT) to the satisfiability of the \emph{maximum} number of XOR constraints (MaxXOR), in particular, of XOR constraints with two or less variables (Max2XOR).

We will consider weighted constraints, noted $\wcla{w}{C=k}$, where $w$ is a rational number that denotes the contribution of the satisfiability of the constraint to the formula. This way, we can translate every binary clause $x\vee y$ as $\{\wcla{1/2}{x=1},\ \wcla{1/2}{y=1},\ \wcla{1/2}{x\xor y=1}\}$, because when $x$ or $y$ are equal to one (i.e. $x\vee  y$ is satisfied), exactly two of the XOR constraints are satisfied, which contributes $1/2+1/2=1$ to the constraint, and when $x$ and $y$ are both equal to zero and the original clause is falsified, none of the XOR constraints are satisfied. We can also translate ternary clauses like $x\vee y\vee z$ as $\{\wcla{1/4}{x=1},$ $\wcla{1/4}{y=1},$ $\wcla{1/4}{z=1},$ $\wcla{1/4}{x\xor y=1},$ $\wcla{1/4}{x\xor z=1},$ $\wcla{1/4}{y\xor z=1},$ $\wcla{1/4}{x\xor y\xor z=1}\}$. In general, any OR clause $x_1\vee \dots\vee x_k$ is equivalent to the set of (weighted) XOR constraints:
\[
\bigcup_{S\subseteq \{1,\dots,k\}} \{\wcla{1/2^{k-1}}{\bigoplus_{i\in S} x_i = 1}\}
\]
This translation allows us to reduce SAT to MaxXOR. However, the reduction is not polynomial, because it translates every clause of size $k$ into $2^k-1$ constraints. In Section~\ref{sec:MaxSAT->Max2XOR}, we describe a polynomial reduction that avoids this exponential explosion on expenses of introducing new variables.

In Section~\ref{sec:OnProofSystems}, we will discuss the possible definition of a proof system for Max2XOR, in the spirit of the MaxSAT resolution which was first defined by~\cite{LarrosaH05}, and proven complete by~\cite{SAT06,AIJ1}.

\section{Preliminaries}

A k-ary \emph{constraint function} is a Boolean function $f:\{0,1\}^k\to\{0,1\}$. 
A \emph{constraint family} is a set $\mathcal F$ of constraint functions (with possibly distinct arities). 
A \emph{constraint}, over variables $V=\{x_1,\dots,x_n\}$ and constraint family $\mathcal F$, is a pair formed by a k-ary constraint function $f\in\mathcal{F}$ and a subset of $k$ variables, noted $f(x_{i_1},\dots,x_{i_k})$, or $f(\vec{x})$ for simplicity.  
A \emph{(weighted) constraint problem} or \emph{(weighted) formula} $P$, over variables $V$ and constraint family $\mathcal F$, is a set of pairs (weight, constraint) over $V$ and $\mathcal{F}$, where the weight is a positive rational number, denoted $P=\{\wcla{w_1}{f_1(x^1_{i_1},\dots,x^1_{i_{k_1}})},\dots,\wcla{w_m}{f_m(x^m_{i_1},\dots,x^m_{i_{k_m}})}\}$. The \emph{weight} of a problem is $\weight(P)=w_1+\cdots+w_m$.

An \emph{assignment} is a function $I:\{x_1,\dots,x_n\}\to\{0,1\}$. We say that an assignment $I$ \emph{satisfies} a constraint $f(x_{i_1},\dots,x_{i_k})$, if $I(f(\vec{x}))=_{\mbox{def}}f(I(x_{i_1}),\dots,I(x_{i_k}))=1$. 
The value of an assignment $I$ for a constraint problem $P=\{\wcla{w_i}{f_i(\vec{x})}\}_{i=1,\dots,m}$, is the sum of the weights of the constraints that this assignment satisfies, i.e. $I(P) = \sum_{i=1}^m w_i\,I(f_i(\vec{x}))$.
We also define the sum of weights of unsatisfied clauses as $\overline I(P) = \sum_{i=1}^m w_i\,(1-I(f_i(\vec{x})))$.
An assignment $I$ is said to be \emph{optimal} for a constraint problem $P$, if it is maximizes $I(P)$. We define this optimum as $\opt(P)=\max_{I} I(P)$. We also define $\cost(P)=\min_I \overline I(P)$, i.e. the minimum sum of weights of falsified constraints, that fulfils $\opt(P)+\cost(P)=\weight(P)$.

For convenience, we generalize constraint problems to multisets, and consider $\{\wcla{w_1}{C},$ $\wcla{w_2}{C}\}$ as equivalent to $\{\wcla{w_1+w_2}{C}\}$. When the weight of a constraint is not explicitly specified, it is assumed to be one.

When using inference rules, we assume that they can be applied with any weight and that premises are decomposed conveniently. For instance, if we have the inference rule $C=0, C=1\vdash \eclause$, and we apply it to the problem $P=\{\wcla{w_1}{A=0},\,\wcla{w_2}{A=1}\}$, assuming $w_1\geq w_2$, we proceed as follows. First, we decompose the first constraint to get $P'=\{\wcla{w_2}{A=0},\,\wcla{w_1-w_2}{A=0},\,\wcla{w_2}{A=1}\}$. Then, we apply the rule with weight $w_2$ to get $P'=\{\wcla{w_2}{\eclause},\,\wcla{w_1-w_2}{A=0}\}$. When applying a rule, we assume that at least one of the premises is not decomposed.

Many optimization problems may be formalized as the problem of finding an optimal assignment for a constraint problem. In the following we define some of them:
\begin{enumerate}
    \item MaxEkSAT is the constraint family defined by the constraint functions of the form $f(x_1,\dots,x_k)= l_1\vee\cdots\vee l_k$, where every $l_i$ may be either $x_i$ or $\neg x_i$.
    \item MaxkSAT is the union of the constraint families MaxEiSAT, for $0\leq i\leq k$, and MaxSAT the union, for every $i\geq 0$.
    \item MaxEkPC$_0$ is the constraint family defined by the constraint function $f(x_1,\dots,x_k)=1-x_1\oplus\cdots\oplus x_k$. Similarly, MaxEkPC$_1$ is defined by $f(x_1,\dots,x_k)=x_1\oplus\cdots\oplus x_k$, and MaxEkXOR or MaxEkPC is the union of MaxEkPC$_0$ and MaxEkPC$_1$.
    \item MaxkXOR or MaxkPC is the union of MaxEiXOR, for $0\leq i\leq k$, and MaxXOR or MaxPC the union, for every $i\geq 0$.
    \item MaxCUT is the same as MaxE2PC$_1$.
\end{enumerate}

In the literature, PC (similarly for PC$_0$ and PC$_1$) use to denote what we denote as MaxE3PC. However, notice that, in this paper, we are interested in the problem Max2XOR, that has got very little attention in the literature, probably because it is quite similar to MaxCUT. Notice that MaxCUT use to be defined as the problem of, given a graph, finding a partition of the vertices that maximizes the number of edges that connects vertices of distinct partition. The problem can be trivially formalized assigning a Boolean variable to each vertex, and a constraint $x\oplus y=1$ to each edge connecting $x$ and $y$.

Next, we define gadgets as a transformation from constraints into sets of weighted constraints (or problems). These transformations can be extended to define transformations of problems into problems.

\begin{definition}
Let $\mathcal{F}_1$ and $\mathcal{F}_2$ be two constraint families.
A (strict) \emph{$(\alpha,\beta)$-gadget} from $\mathcal{F}_1$ to $\mathcal{F}_2$ is a function that, for any constraint $f(\vec{x})$ over $\mathcal{F}_1$ returns a weighted constraint problem $P=\{\wcla{w_i}{g_i(\vec{x},\vec{a})}\}_{i=1,\dots,m}$ over $\mathcal{F}_2$ and variables $\{\vec{x}\}\cup\{\vec{a}\}$, where $\vec{a}$ are fresh variables distinct from $\vec{x}$, such that $\beta=\sum_{i=1}^m w_i$ and, for any assignment $I:\{\vec{x}\}\to\{0,1\}$:
\begin{enumerate}
    \item If $I(f(\vec{x}))=1$, for any extension of $I$ to $I':\{\vec{x}\}\cup\{\vec{a}\}\to\{0,1\}$, $I'(P)\leq \alpha$ and there exist one of such extension with $I'(P) = \alpha$.
    \item If $I(f(\vec{x}))=0$, for any extension of $I$ to $I':\{\vec{x}\}\cup\{\vec{a}\}\to\{0,1\}$, $I'(P)\leq \alpha-1$ and there exist one of such extension with $I'(P) = \alpha-1$.
\end{enumerate}
\end{definition}

\begin{lemma}\label{lem-concatenation}
The composition of a $(\alpha_1,\beta_1)$-gadget from $\mathcal{F}_1$ to $\mathcal{F}_2$ and a $(\alpha_2,\beta_2)$-gadget from $\mathcal{F}_2$ to $\mathcal{F}_3$ results into a $(\beta_1\,(\alpha_2-1)+\alpha_1,\ \beta_1\beta_2)$-gadget from $\mathcal{F}_1$ to~$\mathcal{F}_3$.
\end{lemma}

\begin{proof}
The first gadget multiplies the total weight of constraints by $\beta_1$, and the second by $\beta_2$. Therefore, the composition multiplies it by $\beta=\beta_1\beta_2$.

For any assignment, if the original constraint is falsified, the optimal extension after the first gadget satisfies constraints with a weight of $\alpha_1-1$, and falsifies the rest $\beta_1-(\alpha_1-1)$. The second gadget satisfies constraints for a weight of $\alpha_2-1$ of the falsified plus $\alpha_2$ of the satisfied. Therefore, the composition satisfies constraints with a total weight $(\alpha_2-1)(\beta_1-(\alpha_1-1))+\alpha_2(\alpha_1-1)= \beta_1(\alpha_2-1)+\alpha_1-1$.

If the original constraint is satisfied, the optimal extension after the first gadget satisfies $\alpha_1$ and falsifies the rest $\beta_1-\alpha_1$. After the second gadget the weight of satisfied constraints is $(\alpha_2-1)(\beta_1-\alpha_1)+\alpha_2\alpha_1=\beta_1(\alpha_2-1)+\alpha_1$.

The difference between both situations is one, hence the composition is a gadget, and $\alpha=\beta_1(\alpha_2-1)+\alpha_1$.
\end{proof}

\begin{lemma}
If $P$ is translated into $P'$ using a $(\alpha,\beta)$-gadget, then
\[
\begin{array}{l}
\opt(P')=(\alpha-1)\weight(P)+\opt(P)\\
\cost(P')=(\beta-\alpha)\weight(P)+\cost(P)
\end{array}
\]
\end{lemma}

Previous lemma allows us to obtain an algorithm to maximize $P$, using an algorithm to maximize $P'$. If this algorithm is an approximation algorithm on the optimum, in order to get the minimal error, we are interested in gadgets with minimal $\alpha$. That's because, traditionally, a gadget is called \emph{optimal} if it minimizes $\alpha$. In our case, we will use proof systems that derive empty clauses, i.e. that prove lower bounds for the cost. Therefore, in our case, we will be interested in gadgets that minimize $\beta-\alpha$.

In the case that $P$ is a decision problem, then we can say that $P$ is unsatisfiable if, and only if, $\cost(P') \geq (\beta-\alpha)\weight(P)+1$.

\section{Some Simple Gadgets}\label{sec:simple-gadgets}

As we described in the introduction, there exists a simple translation from OR clauses to XOR constraints that does not need to introduce new variables. In particular, for binary clauses it is as follows.

\begin{lemma}\label{lem:Max2SAT->Max2XOR}
The following set of transformations:
\[
\begin{array}{l@{\ \to\ }l}
{x} & \{\wcla{1}{x=1}\}\\
{\neg x} & \{\wcla{1}{x=0}\}\\
{x\vee y} & \{\wcla{1/2}{x=1},\ \wcla{1/2}{y=1},\ \wcla{1/2}{x\oplus y=1}\}\\
{x\vee \neg y} & \{\wcla{1/2}{x=1},\ \wcla{1/2}{y=0},\ \wcla{1/2}{x\oplus y=0}\}\\
{\neg x\vee y} & \{\wcla{1/2}{x=0},\ \wcla{1/2}{y=1},\ \wcla{1/2}{x\oplus y=0}\}\\
{\neg x\vee\neg y} & \{\wcla{1/2}{x=0},\ \wcla{1/2}{y=0},\ \wcla{1/2}{x\oplus y=1}\}
\end{array}
\]
defines a $(1,3/2)$-gadget from Max2SAT to Max2XOR.
\end{lemma}

Notice that $x_1\oplus\cdots\oplus x_n=0$ and $x_1\oplus\cdots\oplus x_n=1$ are incompatible constraints, hence they can be canceled. Notice also that the order of the variables is irrelevant and $x\oplus x\oplus C = C$. Therefore, we can simplify any PC problem to get an equivalent problem where, for every subset of variables, there is only one constraint.

\begin{example}\label{ex1}
Given the Max2SAT problem:
\[
\left\{
\arraycolsep 3mm
\begin{array}{llll}
\wcla{1}{y} & \wcla{2}{x\vee y}          & \wcla{2}{y\vee \neg z} &\wcla{1}{x\vee z}\\
            & \wcla{1}{\neg x\vee\neg y}      & \wcla{3}{\neg y\vee z} & \wcla{2}{\neg x\vee\neg z}\\
            & \wcla{1}{x\vee\neg y} &                        &\wcla{3}{\neg x\vee z}
\end{array}
\right\}
\]
the gadget described in Lemma~\ref{lem:Max2SAT->Max2XOR}, and its simplification results into the Max2XOR problem:
\[
\left\{
\arraycolsep 5mm
\begin{array}{ll}
     \wcla{1}{x=0}& \wcla{1}{x\oplus y=1}\\
     \wcla{1/2}{y=1}   & \wcla{5/2}{y\oplus z=0}\\
     \wcla{3/2}{z=1}
\end{array}
\right\}
\]
\end{example}

%\section{From Parity Check to MaxCUT}

We can go further and reduce Max2XOR to MaxCUT. Since the constraint family MaxCUT is a subset of the constraint family Max2XOR, the most natural is to make the translation as follows.

\begin{lemma}\label{lem:Max2XOR->MaxCUT}
Given a Max2XOR problem, adding a particular variable, called $\hat 0$, and applying the following transformations:
\[
\begin{array}{l@{\ \to\ }l}
x=1 & \{x\oplus \hat 0 = 1\}\\
x=0 & \{x\oplus a = 1,\ a\oplus \hat 0=1\}\\
x\oplus y = 1 & \{x\oplus y = 1\}\\
x\oplus y = 0 & \{x\oplus b = 1,\ b\oplus y = 1\}
\end{array}
\]
where $a$ and $b$ are fresh variables,
we get a $(2,2)$-gadget from Max2XOR to MaxCUT.
\end{lemma}

Notice that, strictly speaking, the transformations described in Lemma~\ref{lem:Max2XOR->MaxCUT} do not define a gadget, since the auxiliary variable $0$ is not \emph{local} to a constraint, but global to the whole problem.

Alternatively, we can reduce the number of auxiliary variables and constraints if we add two special auxiliary variables called $\hat 0$ and $\hat 1$ and we use the transformations:
\[
\begin{array}{l@{\ \to\ }l}
x=1 & \{x\oplus \hat 0 = 1\}\\
x=0 & \{x\oplus \hat 1 = 1\}\\
& \{\wcla{W}{\hat 0 \oplus \hat 1 = 1}\}\\
x\oplus y = 1 & \{x\oplus y = 1\}\\
x\oplus y = 0 & \{x\oplus b = 1,\ b\oplus y = 1\}
\end{array}
\]
where, if $P$ is the original Max2XOR problem, $\displaystyle W=\min\left\{\sum_{\wcla{w}{x=0}\in P} w,\ \sum_{\wcla{w}{x=1}\in P} w\right\}$. 

The third constraint is added only once and depends on the whole original problem $P$. This special constraint ensures that optimal solutions $I$ of the resulting MaxCUT problem satisfy
$I(\hat 0)\neq I(\hat 1)$. Let $w_i = \sum_{\wcla{w}{x=i}\in P}w$, for $i=0,1$, i.e. $w_i$ is the sum of the weights of constraints of the form $\wcla{w}{x=i}$.  With this alternative transformation, we observe that, from a set of constraints of the form $x=0$ and $x=1$ with total weight $w_0+w_1$, we get a set of constraints of the form $x\oplus \hat 0 = 1$, $x\oplus \hat 1 = 1$ and $\hat 0 \oplus \hat 1 = 1$ with total weight $w_0+w_1+\min\{w_0,w_1\}\leq 3/2(w_0+w_1)$.

\section{From MaxSAT to Max2XOR}\label{sec:MaxSAT->Max2XOR}

An easy way to reduce MaxSAT to Max2XOR is to reduce MaxSAT to Max3SAT, this to Max2SAT, and this to Max2XOR using Lemma~\ref{lem:Max2SAT->Max2XOR}. We can reduce MaxSAT to Max3SAT with the following transformation:
\begin{equation}
\begin{array}{ll}
x_1\vee\cdots\vee x_k \to \{ x_1\vee x_2\vee b_1,\, &\neg b_1\vee x_3\vee b_2,\cdots,\\
&\neg b_{k-4}\vee x_{k-2}\vee b_{k-3},\,\neg b_{k-3}\vee x_{k-1}\vee x_k\}
\end{array}
\label{eq-MaxSAT->Max3SAT}
\end{equation}
that defines a $(k-2,k-2)$-gadget from MaxEkSAT to MaxE3SAT.

Later, we can use the following $(3.5,4)$-gadget from Max3SAT to Max2SAT defined by~Trevisan et al.~\cite{gadgets}:
\begin{equation}
\begin{array}{ll}
\wcla{1}{x_1\vee x_2\vee x_3} \to \{&
\wcla{1/2}{x_1\vee x_3}, 
\ \wcla{1/2}{\neg x_1\vee \neg x_3},\\
&\wcla{1/2}{x_1\vee \neg b},
\ \wcla{1/2}{\neg x_1\vee b},\\
&\wcla{1/2}{x_3\vee \neg b},
\ \wcla{1/2}{\neg x_3\vee b},\\
&\wcla{1}{x_2\vee b}\}
\end{array}
\label{eq-Max3SAT->Max2SAT}
\end{equation}

The concatenation of Trevisan's gadget~(\ref{eq-Max3SAT->Max2SAT}) and the gadget in Lemma~\ref{lem:Max2SAT->Max2XOR} results into the $(2,3)$-gadget described in Figure~\ref{fig:Max3SAT->Max2XOR}. 
\begin{figure}[t]
    \centering
\includegraphics[page=1]{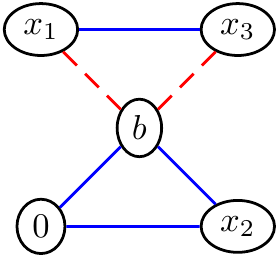}
\ignorar{
    \begin{pspicture}(0,0)(2,2)
    \rput(0,0){\ovalnode{0}{$0$}}
    \rput(0,2){\ovalnode{1}{$x_1$}}
    \rput(2,0){\ovalnode{2}{$x_2$}}
    \rput(2,2){\ovalnode{3}{$x_3$}}
    \rput(1,1){\ovalnode{b}{$b$}}
    \psset{linecolor=blue}
    \ncline{-}{0}{2}
    \ncline{-}{0}{b}
    \ncline{-}{2}{b}
    \ncline{-}{1}{3}
    \psset{linecolor=red, linestyle=dashed}
    \ncline{b}{1}
    \ncline{b}{3}
    \end{pspicture}
    }
    \hspace{30mm}
    $
    \begin{array}[b]{l}
    x_1 \oplus x_3 = 1\\
    x_1 \oplus b = 0\\
    x_3 \oplus b = 0\\
    x_2 \oplus b = 1\\
    b  = 1\\
    x_2 = 1\\
    \end{array}
    $
    \caption{Graphical representation of the $(2,3)$-gadget reducing Max3SAT to Max2XOR and based on \cite{gadgets} gadget from Max3SAT to Max2SAT. All constraints have weight $1/2$. Blue edges represent equal-one constraints, and red dashed edges equal-zero constraints.}
    \label{fig:Max3SAT->Max2XOR}
\end{figure}

\subsection{A Sequential Translation}

Trevisan's gadget~\cite{gadgets} is optimal. However, when concatenated with gadget~(\ref{eq-MaxSAT->Max3SAT}) results into a $(2(k-2),\,3(k-2))$-gadget from MaxEkSAT to Max2XOR, that is not optimal (in the sense that it is not the one with minimal $\alpha$). In the following we describe a better and direct reduction from MaxEkSAT to Max2XOR.

The reduction is based on a function $T^0$ that takes as parameters a SAT clause and a variable. The use of this variable is for technical reasons, to make simpler the recursive definition and proof of Lemma~\ref{lem:T0}. Later (see Theorem~\ref{thm-MaxSAT->PC}), it will be replaced by the constant one.

\begin{definition}
Given a SAT clause $C$ and a variable $b$, the translation function $T^0(C,b)$ is defined recursively as the set of 2XOR clauses
\[
T^0(x_1\vee x_2,b) = \{\wcla{1/2}{x_1\oplus x_2=1},\  \wcla{1/2}{x_1\oplus b=0},\  \wcla{1/2}{b\oplus x_2=0}\}
\]
for binary clauses, and 
\[
T^0(x_1\vee\cdots\vee x_k,b) = T^0(x_1\vee\cdots\vee x_{k-1},b')\cup\{\wcla{1/2}{b'\oplus x_k=1},\ \wcla{1/2}{b'\oplus b=0},\ \wcla{1/2}{b\oplus x_k=0}\}
\]
for $k\geq 3$, where $b'$ is a fresh variable not occurring elsewhere.
\end{definition}

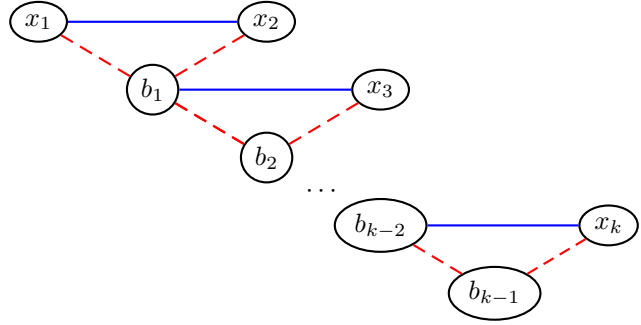
\begin{figure}[t]
\centering
\includegraphics[page=2]{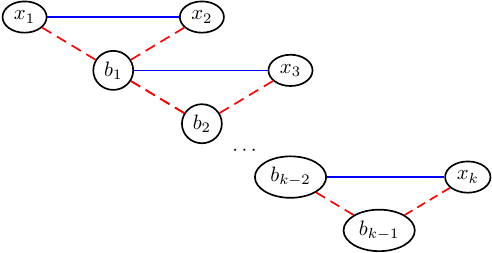}
\ignorar{
\psset{xunit=15mm,yunit=9mm}
\begin{pspicture}(0,0)(5,4)
\rput(0,4){\ovalnode{x1}{$x_1$}}
\rput(1,3){\ovalnode{b1}{$b_1$}}
\rput(2,4){\ovalnode{x2}{$x_2$}}
\rput(2,2){\ovalnode{b2}{$b_2$}}
\rput(3,3){\ovalnode{x3}{$x_3$}}
\rput(3,1){\ovalnode{bn-2}{$b_{k-2}$}}
\rput(4,0){\ovalnode{bn-1}{$b_{k-1}$}}
\rput(5,1){\ovalnode{xn}{$x_k$}}
\rput(2.5,1.5){$\cdots$}
    \psset{linecolor=blue}
    \ncline{-}{x1}{x2}
    \ncline{-}{b1}{x3}
    \ncline{-}{bn-2}{xn}
    \psset{linecolor=red, linestyle=dashed}
    \ncline{-}{b1}{b2}
    \ncline{-}{b1}{b2}
    \ncline{-}{b1}{x1}
    \ncline{-}{b1}{x2}
    \ncline{-}{b2}{x3}
    \ncline{-}{bn-1}{xn}
    \ncline{-}{bn-1}{bn-2}
\end{pspicture}
}
\caption{Graphical representation of $T^0(x_1\vee\cdots\vee x_k, b_{k-1})$. Blue edges represent 2PC$_1$ constraints and red dashed edges 2PC$_0$ constraints, all them with weight $1/2$. }\label{fig-MaxSAT->PC}
\end{figure}

\begin{lemma}\label{lem:T0}
Consider the Max2XOR problem $T^0(x_1\vee\cdots\vee x_k, b_{k-1})$. We have:
\begin{enumerate}
\item For any assignment $I:\{x_1,\dots,x_k\}\to\{0,1\}$,
the extension of the assignment as $I'(b_{i})=I(x_{i+1})$, for $i=1,\dots,k-1$, maximizes the number of satisfied 2XOR constraints, that is $2\,(k-1)$.
\item For any assignment $I:\{x_1,\dots,x_k,b_{k-1}\}\to\{0,1\}$ satisfying $I(b_{k-1})=1$, the extension of the assignment  as
\[
I'(b_{i})=\left\{
\begin{array}{ll}
1 & \mbox{if $I(x_{i+2})=\cdots=I(x_k)=0$}\\[2mm]
I(x_{i+1}) &\mbox{otherwise}
\end{array}
\right.
\]
for $i=1,\dots,k-2$, maximizes the number of satisfied 2XOR constraints, that is $2\,(k-2)$, if $I(x_1)=\cdots=I(x_k)=0$, or $2\,(k-1)$ otherwise.
\end{enumerate}
\end{lemma}

\begin{proof}
For the first statement: the proof is by induction on $k$. 

For $k=2$, we never can satisfy the three constraints $T^0(x_1\vee x_2,b_1) = \{x_1\oplus x_2=1,\ x_1\oplus b_1=0,\ b_1\oplus x_2=0\}$. We can satisfy two of them if we set $b_1$ equal to $x_1$ or to $x_2$ (we choose the second possibility).

For $k>2$, we have:
\[
T^0(x_1\vee\cdots\vee x_k,b_{k-1}) = \begin{array}[t]{l}
T^0(x_1\vee\cdots\vee x_{k-1},b_{k-2})\cup\\[2mm]
\{\wcla{1/2}{b_{k-2}\oplus x_k=1},\ \wcla{1/2}{b_{k-2}\oplus b_{k-1}=0},\ \wcla{1/2}{b_{k-1}\oplus x_k=0}\}
\end{array}
\]
We can never satisfy all three constraints $\{b_{k-2}\oplus x_k=1,\ b_{k-2}\oplus b_{k-1}=0,\ b_{k-1}\oplus x_k=0\}$. By setting $I'(b_{k-1})=I(x_k)$, it does not matter what value $b_{k-2}$ gets, we always satisfy two of them. By induction hypothesis, we know that setting $I'(b_i)=I(x_{i+1})$, for $i=1,\dots,k-2$, we maximize the number of satisfied constraints, i.e. $2(k-2)$, from the rest of constraints in $T^0(x_1\vee\dots\vee x_{k-1},b_{k-2})$. Therefore, we can satisfy a maximum of $2(k-1)$ constraints from $T^0(x_1\vee\dots\vee x_k,b_{k-1})$.

For the second statement: the proof is also by induction on $k$. 

For $k=2$ and $I(b_1)=0$, the number of satisfied constraints is zero, if $I(x_1)=I(x_2)=0$, and two, otherwise.

For $k>2$, there are two cases: 

If $I(x_k)=0$ and $I(b_{k-1})=1$, assigning $I(b_{k-2})=1$ we maximize the number of satisfied constraints involving $b_{k-2}$, satisfying at least two of them, whereas assigning $I(b_{k-2})=0$, we can satisfy at most two of them. Hence, we take the first option. Then, we use the induction hypothesis to prove that we can satisfy other $2\,(k-2)$ constraints if at least some of the $x_i$ is one, for $i=1,\dots,k-1$, or only $2\,(k-3)$, if all them are zero.

If $I(x_k)=1$ and $I(b_{k-1})=1$, it does not matter what value $b_{k-2}$ gets, we satisfy two constraints from $\{b_{k-2}\oplus x_k=1,\ b_{k-2}\oplus b_{k-1}=0,\ b_{k-1}\oplus x_k=0\}$. We use then the first statement of the lemma, and the same assignment for $b$'s, to prove that the proposed assignment $I'$ maximizes the satisfaction of the rest of constraints, i.e. $2\,(k-2)$ of them.
\end{proof}

\begin{theorem}\label{thm-MaxSAT->PC}
The translation of every clause $x_1\vee\cdots\vee x_k$ into $T^0(x_1\vee\cdots\vee x_k, \hat{1})$ defines a $(k-1,3/2(k-1))$-gadget from MaxEkSAT to Max2XOR.
\end{theorem}
\begin{proof}
The soundness of the reduction is based on the second statement of Lemma~\ref{lem:T0}. Assume that we force the variable $\hat{1}$ to be interpreted as one, i.e. $I(\hat{1})=1$, or that, alternatively, we simply replace it by the constant $1$. Any assignment satisfying the clause can extend to satisfy at least $2(k-1)$ 2XOR constraints from $T^0(x_1\vee\cdots\vee x_k, \hat{1})$, with a total weight equal to $k-1$. Any assignment  falsifying it can only be extended to satisfy at most $2(k-2)$ 2XOR constraints, with a total weight $k-2$. Hence, $\alpha = k-1$. The number of 2XOR constraints is $3(k-1)$, all of them with weight $1/2$. Therefore, $\beta=3/2(k-1)$.

We only describe the translation of clause $x_1\vee\cdots\vee x_k$ when all variables are positive. We can easily generalize the transformation, when any of the variables $x$ are negated, simply by recalling that $\neg x\oplus b = k$ is equivalent to $x\oplus b = 1-k$, and $\neg x\oplus \neg x' = k$ is equivalent to $x\oplus x'=k$.
\end{proof}

\begin{remark}
Consider the Max2XOR problem $P=T^0(x_1\vee\cdots\vee x_k,b_{k-1})$.

For any assignment $I:\{x_1,\dots,x_k\}\to\{0,1\}$, its extension, defined by $I(b_i)=I(x_1)\vee\cdots\vee I(x_{i+1})$, for $i=1,\dots k-1$ (similar but not equal to the one proposed in the proof of the first statement of Lemma~\ref{lem:T0}, also) maximizes the number of satisfied constraints of $P$, satisfying $I(P)=k-1$.

For any assignment $I:\{x_1,\dots,x_k,b_{k-1}\}\to\{0,1\}$, where $I(b_{k-1})=1$,
its extension $I(b_i)=I(x_1)\vee\cdots\vee I(x_{i+1})$, for $i=1,\dots k-2$ (similar but not equal to the one proposed in the proof of the second statement of Lemma~\ref{lem:T0}, also) maximizes the number of satisfied constraints of $P$, satisfying $I(P)=k-1$, when $I(x_1\vee\cdots x_k)=1$, and $I(P)=k-2$,  when $I(x_1\vee\cdots x_k)=0$.
\end{remark}

\subsection{A Parallel Translation}

The previous reduction of MaxSAT into Max2XOR can be seen as a kind of \emph{sequential checker}. We can use a sort of \emph{parallel checker}. Before describing the reduction formally, we introduce an example.

\begin{example}\label{ex:parallelTranslation}
The clause $x_1\vee\cdots\vee x_7$ can be reduced to the following Max2XOR problem, where all XOR constraints have weight $1/2$:
\begin{center}
\includegraphics[page=3]{main-pics.pdf}
\ignorar{
\psset{xunit=15mm,yunit=15mm}
\begin{pspicture}(0,-0.3)(6,3.3)
\rput(0,3){\ovalnode{x1}{$x_1$}}
\rput(1,3){\ovalnode{x2}{$x_2$}}
\rput(2,3){\ovalnode{x3}{$x_3$}}
\rput(3,3){\ovalnode{x4}{$x_4$}}
\rput(4,3){\ovalnode{x5}{$x_5$}}
\rput(5,3){\ovalnode{x6}{$x_6$}}
\rput(6.5,2){\ovalnode{x7}{$x_7$}}
\rput(0.5,2){\ovalnode{b1}{$b_1$}}
\rput(2.5,2){\ovalnode{b2}{$b_2$}}
\rput(4.5,2){\ovalnode{b3}{$b_3$}}
\rput(1.5,1){\ovalnode{b4}{$b_4$}}
\rput(5.5,1){\ovalnode{b5}{$b_5$}}
\rput(3.5,0){\ovalnode{1}{$\hat{1}$}}
    \psset{linecolor=blue}
    \ncline{-}{x1}{x2}
    \ncline{-}{x3}{x4}
    \ncline{-}{x5}{x6}
    \ncline{-}{b1}{b2}
    \ncline{-}{b3}{x7}
    \ncline{-}{b4}{b5}
    \psset{linecolor=red, linestyle=dashed}
    \ncline{-}{x1}{b1}
    \ncline{-}{x2}{b1}
    \ncline{-}{x3}{b2}
    \ncline{-}{x4}{b2}
    \ncline{-}{x5}{b3}
    \ncline{-}{x6}{b3}
    \ncline{-}{b1}{b4}
    \ncline{-}{b2}{b4}
    \ncline{-}{b3}{b5}
    \ncline{-}{x7}{b5}
    \ncline{-}{b4}{1}
    \ncline{-}{b5}{1}

\end{pspicture}
}
\end{center}

\ignorar{
%%%
%%% TAMBIEN VA A LA VERSION DE REVISTA
%%%
By negating some of the $b$'s and $x$'s we can get an equivalen 2PC problem with only blue edges:

\begin{center}
\psset{xunit=15mm,yunit=15mm}
\begin{pspicture}(0,-0.3)(6,3.3)
\rput(0,3){\ovalnode{x1}{$\neg x_1$}}
\rput(1,3){\ovalnode{x2}{$\neg x_2$}}
\rput(2,3){\ovalnode{x3}{$\neg x_3$}}
\rput(3,3){\ovalnode{x4}{$\neg x_4$}}
\rput(4,3){\ovalnode{x5}{$\neg x_5$}}
\rput(5,3){\ovalnode{x6}{$\neg x_6$}}
\rput(6.5,2){\ovalnode{x7}{$x_7$}}
\rput(0.5,2){\ovalnode{b1}{$b_1$}}
\rput(2.5,2){\ovalnode{b2}{$b_2$}}
\rput(4.5,2){\ovalnode{b3}{$b_3$}}
\rput(1.5,1){\ovalnode{b4}{$\neg b_4$}}
\rput(5.5,1){\ovalnode{b5}{$\neg b_5$}}
\rput(3.5,0){\ovalnode{1}{$\hat{1}$}}
    \psset{linecolor=blue}
    \ncline{-}{x1}{x2}
    \ncline{-}{x3}{x4}
    \ncline{-}{x5}{x6}
    \ncline{-}{b1}{b2}
    \ncline{-}{b1}{b4}
    \ncline{-}{b2}{b4}
    \ncline{-}{x1}{b1}
    \ncline{-}{x2}{b1}
    \ncline{-}{x3}{b2}
    \ncline{-}{x4}{b2}
    \ncline{-}{x5}{b3}
    \ncline{-}{x6}{b3}
    \ncline{-}{x7}{b5}
    \ncline{-}{b5}{1}
%    \psset{linecolor=red, linestyle=dashed}
    \ncline{-}{b3}{x7}
    \ncline{-}{b4}{b5}
    \ncline{-}{b3}{b5}
    \ncline{-}{b4}{1}
\end{pspicture}
\end{center}
}
\end{example}

\begin{definition}
Given a clause $x_1\vee\cdots x_k$ and a variable $b$, their translation into Max2XOR constraints is defined recursively and non-deterministically as:
$$
T(x_1\vee\cdots\vee x_k,\ b) =
\left\{
\begin{tabular}{l@{\ }p{0.25\textwidth}}
$\{x_1\oplus x_2=1, x_1\oplus b=0, x_2\oplus b=0\}$
& if $k=2$\\[3mm]
\arraycolsep 0mm
$\begin{array}[t]{l}
\{b'\oplus x_k=1, b'\oplus b=0, x_k\oplus b=0\}\ \cup\\
T(x_1\vee\cdots\vee x_{k-1},\ b')
\end{array}$
&if $k\geq 3$, where $b'$ is a fresh variable\\[7mm]
\arraycolsep 0mm
$\begin{array}[t]{l}
\{b'\oplus x_1=1, b'\oplus b=0, x_1\oplus b=0\}\ \cup\\
T(x_2\vee\cdots\vee x_{k},\ b')
\end{array}$
&if $k\geq 3$, where $b'$ is a fresh variable\\[7mm]
\arraycolsep 0mm
$\begin{array}[t]{l}
\{b'\oplus b''=1, b'\oplus b=0, b''\oplus b=0\}\ \cup\\
\ T(x_1\vee\cdots\vee x_r,\ b')\ \cup\\
\ T(x_{r+1}\vee\cdots\vee x_k,\ b'')
\end{array}$
&if $k\geq 4$, where $b'$ and $b''$ are fresh variables, and $2\leq r \leq k-2$
\end{tabular}
\right.
$$
where all XOR constraints have weight $1/2$.
\end{definition}

Notice that the definition of $T$ is not deterministic. When $k=3$ we have two possible translations (applying the second or third case), and when $k\geq 4$ we can apply the second, third and forth cases, with distinct values of $r$, getting as many possible translations as possible binary trees with $k$ leaves.

Notice also that the definition of $T^0$ corresponds to the definition of $T$ where we only apply the first and second cases. Therefore, $T$ is a generalization of $T^0$.

\begin{lemma}\label{lem:parallel}
Consider the set of 2XOR constraints of $T(x_1\vee\cdots\vee x_k,b)$ without weights. We have:
\begin{enumerate}
    \item There are $3(k-1)$ constraints.
    \item Any assignment satisfies at most $2(k-1)$ of them.
    \item Any assignment $I:\{x_1,\dots,x_k\}\to\{0,1\}$ can be extended to an optimal assignment satisfying $I(b)=I(x_1\vee\cdots\vee x_k)$ and $2(k-1)$ constraints.
    \item Any assignment $I$ satisfying $I(x_1\vee\cdots\vee x_k)=0$ and $I(b)=1$ can be extended to an optimal assignment satisfying $2(k-2)$ constraints.
\end{enumerate}
\end{lemma}

\begin{proof}
The first statement is trivial.

For the second, notice that for, each one of the $k-1$ \emph{triangles} of the form $\{a\oplus b=1, a\oplus c=0, b\oplus c=0\}$,  any assignment can satisfy at most two of the constraints of each triangle.

The third statement is proved by induction. The base case, for binary clauses is trivial. For the induction case, if we apply the second or third options of the definition of $T$, the proof is similar to Lemma~\ref{lem:T0}. In the forth case, assume by induction that there is an assignment extending $I:\{x_1,\dots,x_r\}\to\{0,1\}$ that verifies $I(b')=I(x_1\vee\cdots\vee x_r)$ and satisfies $2(r-2)$ constraints of $T(x_1\vee\cdots\vee x_r,b')$. Similarly, there is an assignment extending $I:\{x_{r+1},\dots,x_k\}\to\{0,1\}$ that verifies $I(b'')=I(x_{r+1}\vee\cdots\vee x_k)$ and satisfies $2((k-r)-2)$ constraints of $T(x_{r+1}\vee\cdots\vee x_k,b'')$. Both assignments do not share variables, therefore, they can be combined into a single assignment. This assignment can be extended with $I(b)=b'\vee b''$. This ensures that it will satisfy two of the constraints from $\{b'\oplus b''=1, b'\oplus b=0, b''\oplus b=0\}$. Therefore, it verifies $I(b)=I(x_1\vee\cdots\vee x_k)$ and satisfies $2(r-2)+2((k-r)-2)+2=2(k-1)$ constraints of $T(x_1\vee\cdots\vee x_k,b)$. Since this is the maximal number of constraints we can satisfy, this assignment is optimal.

The forth statement is also proved by induction. The base case, for binary clauses is trivial. For the induction case, consider only the application of the forth case of the definition of $T$ (the other cases are quite similar). We have to consider 3 possibilities:

If we extend $I(b')=I(b'')=1$, by induction, we can extend $I$ to satisfy $2(r-2)$ constraints of $T(x_1\vee\cdots\vee x_r,b')$ and $2((k-r)-2)$ constraints of $T(x_{r+1}\vee\cdots\vee x_k,b'')$. It satisfies $2$ constraints from $\{b'\oplus b''=1, b'\oplus b=0, b''\oplus b=0\}$, hence a total of $2(k-3)$ constraints.

If we extend $I(b')=I(b'')=0$, by the third statement of this lemma, we can extend $I$ to satisfy $2(r-1)$ constraints of $T(x_1\vee\cdots\vee x_r,b')$ and $2((k-r)-1)$ constraints of $T(x_{r+1}\vee\cdots\vee x_k,b'')$. It does not satisfy any constraints from $\{b'\oplus b''=1, b'\oplus b=0, b''\oplus b=0\}$, hence a total of $2(k-2)$ constraints.

Finally, If we extend $I(b')=1$ and $I(b'')=0$ (or vice versa), by induction, we can extend $I$ to satisfy $2(r-2)$ constraints of $T(x_1\vee\cdots\vee x_r,b')$ and by the third statement $2((k-r)-2)$ constraints of $T(x_{r+1}\vee\cdots\vee x_k,b'')$. It also satisfy $2$ constraints from $\{b'\oplus b''=1, b'\oplus b=0, b''\oplus b=0\}$, hence a total of $2(k-2)$ constraints.
\end{proof}

\begin{theorem}\label{thm:parallel}
The translation of every clause $x_1\vee\cdots\vee x_k$
into $T(x_1\vee\cdots\vee x_k,\hat{1})$ 
define a $(k-1, 3/2(k-1))$-gadget from MaxEkSAT into Max2XOR.
\end{theorem}
\begin{proof}
The theorem is a direct consequence of the two last statements of Lemma~\ref{lem:parallel}.
\end{proof}

If we compare Theorems~\ref{thm-MaxSAT->PC} and~\ref{thm:parallel}, we observe that both establish the same values for $\alpha=k-1$ and $\beta=3/2(k-1)$ in the $(\alpha,\beta)$-gadget, and both introduce the same number $(k-2)$ of fresh variables. Therefore, a priory, the parallel translation is a generalization of the sequential translation, but it does not have a clear (theoretical) advantage. However, if we think on industrial instances with high modularity, it makes sense to generate a tree (like the one in Example~\ref{ex:parallelTranslation}), where distant variables in the original SAT problem, are located distant in the tree. For instance, using the algorithms for analyzing the modularity of industrial SAT instances (see \cite{SAT12,JAIR19}) we can determine how distant are two variables in the original SAT problem, or even better, we can obtain a dendrogram of these variables. Then, we can use this distance or the dendrogram to construct a better translation, i.e. a tree where distant variables are kept in distinct branches. This will have the effect of increasing the modularity of the problem, hence of decreasing its difficulty. We think that, using these trees for industrial SAT instances, we will obtain better experimental results than with the sequential translation.

%%%%%%%%%%%%%%%%%%%%%%%%%%%%%%%%%%%%%%%%%%%%%%%%%%%%%%%%%%%%%%%%%%%%%%%%%%%%%%%%%%%%%%%%%%%%%%%%%%%%%%%%%
\section{On Proof Systems for Max2XOR}\label{sec:OnProofSystems}

We can define a proof system for MaxSAT based on the MaxSAT resolution rule~\cite{LarrosaH05,SAT06,AIJ1}:
\[
\begin{array}{c}
x\vee A\\
\neg x\vee B\\
\hline
A\vee B\\
x\vee A\vee \neg B\\
\neg x\vee B\vee \neg A
\end{array}
\]
where premises are replaced by the conclusions. This rule generalizes the traditional SAT resolution rule $x\vee A,\neg x\vee B\vdash A\vee B$, but not only preserving satisfiability, but also the number of unsatisfied clauses, for any interpretation of the variables. In other words, we say that an inference rule $P\vdash C$ is \emph{sound} if $\overline I(P)=\overline I(C)$, for any assignment $I$ of the variables.
This implies that, if we can derive $P\vdash \{\wcla{m}{\eclause}\}\cup P'$, where $P'$ is satisfiable, then $m$ is equal to the minimum sum of weights of unsatisfied clauses, i.e. $\cost(P) = m$. Moreover, this rule is \emph{complete}, in the sense that when $m$ is the minimum number of unsatisfiable clauses $m=\cost(P)$, we can construct a MaxSAT refutation proof of the form $P\vdash \{\wcla{m}{\eclause}\}\cup P'$, where $P'$ is satisfiable.

Gadget transformations can also be written in the form of inference rules. We know that a $(\alpha,\beta)$-gadget from $P$ to $P'$ satisfies $\cost(P')=(\beta-\alpha)\weight(P)+\cost(P)$. Therefore, unless $\alpha=\beta$, the weight of unsatisfiable constraints is shifted. If we want to write an inference rule that preserves the sum of weights of unsatisfied constraints, we have to add an empty clause with weight $(\beta-\alpha)\weight(P)$ to the premises, or an empty clause with the weight negated to the conclusions (on the well-understanding that this empty clause will be canceled by other proved empty clauses). For instance, the Max2SAT to Max2XOR $(1,3/2)$-gadget that transforms $a\vee b$ to $\{\wcla{1/2}{a=1},\ \wcla{1/2}{b=1},\ \wcla{1/2}{a\xor b=1}\}$, and the $(1,7/4)$-gadget from Max3SAT to Max3XOR, can be written as:
\begin{equation}
\begin{array}{c}
     a\vee b\\
     \hline
     \wcla{-1/2}{\eclause}\\
     \wcla{1/2}{a=1}\ \
     \wcla{1/2}{b=1}\\
     \wcla{1/2}{a\xor b=1}
\end{array}
\hspace{10mm}
\begin{array}{c}
     a\vee b\vee c\\
     \hline
     \wcla{-3/4}{\eclause}\\
     \wcla{1/4}{a=1}\ \
     \wcla{1/4}{b=1}\ \
     \wcla{1/4}{c=1}\\
     \wcla{1/4}{a\xor b=1}\ \
     \wcla{1/4}{b\xor c=1}\ \
     \wcla{1/4}{a\xor c=1}\\
     \wcla{1/4}{a\xor b\xor c=1}
\end{array}
\label{eq:gadgets}
\end{equation}

Something similar can be done with the MaxSAT to Max2XOR gadget described in Section~\ref{sec:MaxSAT->Max2XOR}.

\subsection{A Polynomial and Incomplete Proof System for Max2XOR}

With respect to a proof system for Max2XOR, our first step is to use a mixture of XOR and SAT constraints. Since we are able to reduce SAT to XOR constraints of size 2, we will only consider XOR constraints of this size. The proof system is described by the rules of Figure~\ref{fig-rules} to which we refer as Max2XOR proof system. Notice that, in all rules, the SAT clauses prevent the assignments that falsify the two premises. Notice also that, in this first version, none of the rules transforms the SAT constraints, and the SAT constraints are never resolved. Therefore, in the best case, the application of the rules to a XOR problem finishes with a set of SAT constraints.

\begin{figure}[ht]
\[
\begin{array}{c}
x\xor a = 0\\
x\xor b = 0\\
\hline
a\xor b = 0\\
\wcla{2}{x\vee \neg a\vee \neg b}\\
%x\vee \neg a\vee \neg b\\
\wcla{2}{\neg x\vee a\vee b}\\
\end{array}
\hspace{10mm}
\begin{array}{c}
x\xor a = 0\\
x\xor b = 1\\
\hline
a\xor b = 1\\
\wcla{2}{x\vee \neg a\vee b}\\
\wcla{2}{\neg x\vee a\vee \neg b}
\end{array}
\hspace{10mm}
\begin{array}{c}
x\xor a = 1\\
x\xor b = 1\\
\hline
a\xor b = 0\\
\wcla{2}{x\vee a\vee b}\\
\wcla{2}{\neg x\vee \neg a\vee \neg b}
\end{array}
\hspace{10mm}
\begin{array}{c}
C = 0\\
C = 1\\
\hline
\eclause
\end{array}
\]

\[
\begin{array}{c}
x = 0\\
x\xor a = 0\\
\hline
a = 0\\
\wcla{2}{\neg x\vee a}
\end{array}
\hspace{10mm}
\begin{array}{c}
x = 0\\
x\xor a = 1\\
\hline
a = 1\\
\wcla{2}{\neg x\vee \neg a}
\end{array}
\hspace{10mm}
\begin{array}{c}
x = 1\\
x\xor a = 0\\
\hline
a = 1\\
\wcla{2}{x\vee \neg a}
\end{array}
\hspace{10mm}
\begin{array}{c}
x = 1\\
x\xor a = 1\\
\hline
a = 0\\
\wcla{2}{x\vee a}
\end{array}
\]
\caption{Max2XOR resolution rules. The SAT clauses in some of the rules ensure that, when the assignment falsifies the two premises, the number of falsified clauses in the conclusion is the same.}\label{fig-rules}
\end{figure}

\begin{lemma}\label{lem:soundness}
The rules in the Max2XOR proof system are sound.
\end{lemma}
\begin{proof}
It would suffice to build the truth table and check that indeed the weight of unsatisfied clauses in the premises is preserved in the conclusions for every interpretation. 

We alternatively prove the soundness by the application of the MaxSAT resolution rule which we know preserves the weight of the unsatisfied clauses. For example, let's show below the steps from left to right applied on the premises $x \xor a = 0, x \xor b = 0$. The first and the last steps correspond to direct translations between XOR constraints and clauses. The other two steps are applications of the MaxSAT resolution rule. If we fold the ternary clauses we get the conclusions $a \xor b = 0$, $\wcla{2}{x\vee \neg a\vee \neg b}$, $\wcla{2}{\neg x\vee a\vee b}$. The rest of the cases in figure \ref{fig-rules} can be proven in a similar way.
%\begin{figure}[ht]
\[
\begin{array}{c}
\textcolor{blue}{x\xor a = 0}\\
\textcolor{blue}{x\xor b = 0}\\
\hline
x \vee \neg a\\
\neg x \vee a\\ 
x \vee \neg b\\
\neg x \vee b\\ 
\end{array}
\hspace{10mm}
\begin{array}{c}
x \vee \neg a\\
\neg x \vee b\\ 
\hline
\neg a \vee b\\
\textcolor{blue}{x \vee \neg a \vee \neg b}\\
\textcolor{blue}{\neg x \vee a \vee b}
\end{array}
\hspace{10mm}
\begin{array}{c}
x \vee \neg b\\
\neg x \vee a\\ 
\hline
\neg b \vee a\\
\textcolor{blue}{x \vee \neg a \vee \neg b}\\
\textcolor{blue}{\neg x \vee a \vee b}
\end{array}
\hspace{10mm}
\begin{array}{c}
a \vee \neg b\\
\neg a \vee b\\
\hline
\textcolor{blue}{a \xor b = 0}\\
\end{array}
\]
%\caption{Max2XOR resolution rule soundness.}\label{fig-soundeness-rules}
%\end{figure}
\end{proof}

%\hl{Dar la version}

%\jordi{Explicar cuantas cajas hay que generar cuando traducimos de SAT a Max2XOR y queremos demostrar que la formula orifinal es UNSAT.}

\begin{theorem}
For any weighted Max2XOR formula $P$, if inference rules of Figure~\ref{fig-rules} allows us to derive $P\vdash \{    \wcla{m}{\eclause}\}\cup P'$, then the minimal unsatisfiable weight of $P$ is equal to the minimal unsatisfiable weight of $P'$ plus $m$, i.e. $\cost(P) = m+\cost(P')$.

Moreover, the length of any of these derivations is linearly bounded.
\end{theorem}
\begin{proof}
The first statement is a direct consequence of Lemma~\ref{lem:soundness}.

For the second statement, notice that all rule applications remove at least one XOR constraint.
\end{proof}

Notice that we cannot prove a completeness result for this proof system. In other words, unless $P=NP$, we cannot prove that any Max2XOR formula $P$, admits a derivation $P\vdash \{    \wcla{m}{\eclause}\}\cup P'$, where $P'$ is satisfiable. Therefore, the proof system is polynomial, but incomplete. \ignorar{However, in the next section, we show how it is able to prove the unsatisfiability of the pigeon-hole principle for $3$ pigeons and $2$ holes.}

%\hl{Notice that there is a linear application of the rule such that \emph{eliminates} (resolves till saturation) all the variables.} 

\subsection{Managing SAT constraints in Max2XOR Resolution rules}

In our second version of the Max2XOR proof system, in contrast to the first version, we do not ignore the SAT constraints produced by the Max2XOR resolution rules. 

We can simply use our gadget for SAT into Max2XOR to translate the SAT clauses produced by the Max2XOR resolution rules. We can also come up with a more compact representation with fewer fresh variables. For example, for the resolution rule on premises $x\xor a = 0, x\xor b = 0$ we could propose the following resolution rule:
\[
\begin{array}{c}
x\xor a = 0\\
x\xor b = 0\\
\hline
a\xor b = 1\\
\wcla{2}{x\xor y=0}\\
\wcla{2}{a\xor y=0}\\
\wcla{2}{b\xor y=0}
\end{array}
\mbox{\hspace{10mm}as an alternative to:\hspace{10mm} }
\begin{array}{c}
x\xor a = 0\\
x\xor b = 0\\
\hline
a\xor b = 0\\
\wcla{2}{x\vee \neg a\vee \neg b}\\
%x\vee \neg a\vee \neg b\\
\wcla{2}{\neg x\vee a\vee b}\\
\end{array}
\]
where $y$ is a fresh variable. (Notice that $a\xor b = 0$ is changed by $a\xor b = 1$ in the conclusions). The second and third rules in Figure~\ref{fig-rules} can also be modified in a similar way.

Notice that whether we use the gadget or the more compact version provided above, we will introduce auxiliary variables. If we translate the SAT constraints with gadgets in~(\ref{eq:gadgets}), we can avoid the introduction of fresh variables. However, in this case, the conclusions are transformed in exactly the same set of XOR constraints occurring as premises. In the following, we explain further why we do need to introduce auxiliary variables while in the MaxSAT resolution rules this is not needed.

%It seems complicated to construct a complete proof system using those kind of inference rules that introduce fresh variables. 

When dealing with XOR constraints, things are more complicated than with SAT clauses. The basic reason is that, given a function from a set of $n$ Boolean values to positive rational numbers $f:\{0,1\}^n\to \mathbb{Q}^+$, representing the sum of the weights of satisfied clauses for a given assignment, we may have several ways to represent it as a SAT problem $P$, such that $f(I)=I(P)$. For instance, the function $f(a,b,c)=2+a\,(1-c)+c\,b$ may be represented as $P_1=\{\top,\ a\vee c,\ b\vee\neg c\}$, or as $P_2=\{a\vee b,\ c\vee a\vee\neg b,\ \neg c\vee b\vee\neg a\}$.\footnote{Notice the similarity between these problems and the premises and conclusions of the MaxSAT resolution rule.} Basically, because there are $2^n$ distinct assignments to the variables and $3^n$ distinct possible clauses (notice that a variable can appear with positive or negative polarity, or not appear in a constraint). However, using XOR constraints, since there are at most $2^n$ possible clauses (since negative polarities are not used in XOR constraints), there is only one way to represent one of such functions. In our example, $f$ can only be represented as $P_3=\{\wcla{3/2}\top,\, \wcla{1/2}{a=1},\, \wcla{1/2}{b=1},\, \wcla{1/2}{x\xor a=1},\, \wcla{1/2}{x\xor b=0}\}$. The same applies if we want to preserve the sum of the weights of unsatisfied clauses, as in the MaxSAT resolution rule. This means that the only way of writing an inference rule for Max2XOR, in the style of MaxSAT resolution, preserving the sum of the weights of unsatisfied clauses, is by introducing fresh variables. 

%For instance, we could propose the rule
%\[
%\begin{array}{c}
%a\xor b = 0\\
%a\xor c = 0\\
%b\xor c = 0\\
%\hline
%\wcla{2}{x\xor a=0}\\
%\wcla{2}{x\xor b=0}\\
%\wcla{2}{x\xor c=0}\\
%\end{array}
%\]
%where $x$ is a fresh variable. 

%It seems complicated to construct a complete proof system using those kind of inference rules that introduce fresh variables. 

Another important question is how to use the incomplete proof system on practice. Suppose that, starting with $P$ we get $\{    \wcla{m}{\eclause}\}\cup P'$, and no inference rule is applicable. If $P'$ is satisfiable, \ignorar{as in the example of the following section,} we are finished (in this case, the incomplete proof system has been able to refute the formula). Otherwise, we can translate all introduced SAT clauses into 2XOR constraints, using our gadget, and repeat the process. In order to shorten this process, we have to try to get a formula $P'$ that has a small cost, or in other words, a value $m$ for the weight of $\eclause$ as high as possible.

\section{Conclusions}

We have presented a new gadget from SAT into Max2XOR. We have also presented a first version of an incomplete proof system for Max2XOR, and pointed out some ideas to get a complete version. The final definition of this system, and the proof of completeness is left as further work. The practical implementation of a Max2XOR solver based on these ideas is also a future work.

This opens the avenue to explore the practical power of a new proof system for SAT, where the clauses in the input SAT instance are translated into Max2XOR and the Max2XOR resolution rules are applied to them.

\bibliography{bibliography.bib}

\begin{thebibliography}{10}

\bibitem{JAIR19}
Carlos Ans{\'{o}}tegui, Maria~Luisa Bonet, Jes{\'{u}}s Gir{\'{a}}ldez{-}Cru,
  Jordi Levy, and Laurent Simon.
\newblock Community structure in industrial {SAT} instances.
\newblock {\em J. Artif. Intell. Res.}, 66:443--472, 2019.
\newblock URL: \url{https://doi.org/10.1613/jair.1.11741}, \href
  {http://dx.doi.org/10.1613/jair.1.11741} {\path{doi:10.1613/jair.1.11741}}.

\bibitem{SAT12}
Carlos Ans{\'{o}}tegui, Jes{\'{u}}s Gir{\'{a}}ldez{-}Cru, and Jordi Levy.
\newblock The community structure of {SAT} formulas.
\newblock In Alessandro Cimatti and Roberto Sebastiani, editors, {\em Theory
  and Applications of Satisfiability Testing - {SAT} 2012 - 15th International
  Conference, Trento, Italy, June 17-20, 2012. Proceedings}, volume 7317 of
  {\em Lecture Notes in Computer Science}, pages 410--423. Springer, 2012.
\newblock URL: \url{https://doi.org/10.1007/978-3-642-31612-8\_31}, \href
  {http://dx.doi.org/10.1007/978-3-642-31612-8\_31}
  {\path{doi:10.1007/978-3-642-31612-8\_31}}.

\bibitem{BaumgartnerM00}
Peter Baumgartner and Fabio Massacci.
\newblock The taming of the {(X)OR}.
\newblock In John~W. Lloyd, Ver{\'{o}}nica Dahl, Ulrich Furbach, Manfred
  Kerber, Kung{-}Kiu Lau, Catuscia Palamidessi, Lu{\'{\i}}s~Moniz Pereira,
  Yehoshua Sagiv, and Peter~J. Stuckey, editors, {\em Computational Logic -
  {CL} 2000, First International Conference, London, UK, 24-28 July, 2000,
  Proceedings}, volume 1861 of {\em Lecture Notes in Computer Science}, pages
  508--522. Springer, 2000.
\newblock URL: \url{https://doi.org/10.1007/3-540-44957-4\_34}, \href
  {http://dx.doi.org/10.1007/3-540-44957-4\_34}
  {\path{doi:10.1007/3-540-44957-4\_34}}.

\bibitem{SAT06}
Maria~Luisa Bonet, Jordi Levy, and Felip Many{\`{a}}.
\newblock A complete calculus for {Max-SAT}.
\newblock In Armin Biere and Carla~P. Gomes, editors, {\em Theory and
  Applications of Satisfiability Testing - {SAT} 2006, 9th International
  Conference, Seattle, WA, USA, August 12-15, 2006, Proceedings}, volume 4121
  of {\em Lecture Notes in Computer Science}, pages 240--251. Springer, 2006.
\newblock URL: \url{https://doi.org/10.1007/11814948\_24}, \href
  {http://dx.doi.org/10.1007/11814948\_24} {\path{doi:10.1007/11814948\_24}}.

\bibitem{AIJ1}
Maria~Luisa Bonet, Jordi Levy, and Felip Many{\`{a}}.
\newblock Resolution for {Max-SAT}.
\newblock {\em Artif. Intell.}, 171(8-9):606--618, 2007.
\newblock URL: \url{https://doi.org/10.1016/j.artint.2007.03.001}, \href
  {http://dx.doi.org/10.1016/j.artint.2007.03.001}
  {\path{doi:10.1016/j.artint.2007.03.001}}.

\bibitem{Chen09}
Jingchao Chen.
\newblock Building a hybrid {SAT} solver via conflict-driven, look-ahead and
  {XOR} reasoning techniques.
\newblock In Oliver Kullmann, editor, {\em Theory and Applications of
  Satisfiability Testing - {SAT} 2009, 12th International Conference, {SAT}
  2009, Swansea, UK, June 30 - July 3, 2009. Proceedings}, volume 5584 of {\em
  Lecture Notes in Computer Science}, pages 298--311. Springer, 2009.
\newblock URL: \url{https://doi.org/10.1007/978-3-642-02777-2\_29}, \href
  {http://dx.doi.org/10.1007/978-3-642-02777-2\_29}
  {\path{doi:10.1007/978-3-642-02777-2\_29}}.

\bibitem{HeuleDZM04}
Marijn Heule, Mark Dufour, Joris van Zwieten, and Hans van Maaren.
\newblock March{\_}eq: Implementing additional reasoning into an efficient
  look-ahead {SAT} solver.
\newblock In Holger~H. Hoos and David~G. Mitchell, editors, {\em Theory and
  Applications of Satisfiability Testing, 7th International Conference, {SAT}
  2004, Vancouver, BC, Canada, May 10-13, 2004, Revised Selected Papers},
  volume 3542 of {\em Lecture Notes in Computer Science}, pages 345--359.
  Springer, 2004.
\newblock URL: \url{https://doi.org/10.1007/11527695\_26}, \href
  {http://dx.doi.org/10.1007/11527695\_26} {\path{doi:10.1007/11527695\_26}}.

\bibitem{HeuleM04}
Marijn Heule and Hans van Maaren.
\newblock Aligning {CNF-} and equivalence-reasoning.
\newblock In {\em {SAT} 2004 - The Seventh International Conference on Theory
  and Applications of Satisfiability Testing, 10-13 May 2004, Vancouver, BC,
  Canada, Online Proceedings}, 2004.
\newblock URL: \url{http://www.satisfiability.org/SAT04/programme/72.pdf}.

\bibitem{LaitinenJN11}
Tero Laitinen, Tommi~A. Junttila, and Ilkka Niemel{\"{a}}.
\newblock Equivalence class based parity reasoning with {DPLL(XOR)}.
\newblock In {\em {IEEE} 23rd International Conference on Tools with Artificial
  Intelligence, {ICTAI} 2011, Boca Raton, FL, USA, November 7-9, 2011}, pages
  649--658. {IEEE} Computer Society, 2011.
\newblock URL: \url{https://doi.org/10.1109/ICTAI.2011.103}, \href
  {http://dx.doi.org/10.1109/ICTAI.2011.103}
  {\path{doi:10.1109/ICTAI.2011.103}}.

\bibitem{LaitinenJN12sat12}
Tero Laitinen, Tommi~A. Junttila, and Ilkka Niemel{\"{a}}.
\newblock Conflict-driven {XOR}-clause learning.
\newblock In Alessandro Cimatti and Roberto Sebastiani, editors, {\em Theory
  and Applications of Satisfiability Testing - {SAT} 2012 - 15th International
  Conference, Trento, Italy, June 17-20, 2012. Proceedings}, volume 7317 of
  {\em Lecture Notes in Computer Science}, pages 383--396. Springer, 2012.
\newblock URL: \url{https://doi.org/10.1007/978-3-642-31612-8\_29}, \href
  {http://dx.doi.org/10.1007/978-3-642-31612-8\_29}
  {\path{doi:10.1007/978-3-642-31612-8\_29}}.

\bibitem{LaitinenJN12ictai12}
Tero Laitinen, Tommi~A. Junttila, and Ilkka Niemel{\"{a}}.
\newblock Extending clause learning {SAT} solvers with complete parity
  reasoning.
\newblock In {\em {IEEE} 24th International Conference on Tools with Artificial
  Intelligence, {ICTAI} 2012, Athens, Greece, November 7-9, 2012}, pages
  65--72. {IEEE} Computer Society, 2012.
\newblock URL: \url{https://doi.org/10.1109/ICTAI.2012.18}, \href
  {http://dx.doi.org/10.1109/ICTAI.2012.18} {\path{doi:10.1109/ICTAI.2012.18}}.

\bibitem{LarrosaH05}
Javier Larrosa and Federico Heras.
\newblock Resolution in {Max-SAT} and its relation to local consistency in
  weighted csps.
\newblock In {\em IJCAI}, pages 193--198, 2005.

\bibitem{Li00prima}
Chu~Min Li.
\newblock Equivalency reasoning to solve a class of hard {SAT} problems.
\newblock {\em Inf. Process. Lett.}, 76(1-2):75--81, 2000.
\newblock URL: \url{https://doi.org/10.1016/S0020-0190(00)00126-5}, \href
  {http://dx.doi.org/10.1016/S0020-0190(00)00126-5}
  {\path{doi:10.1016/S0020-0190(00)00126-5}}.

\bibitem{Li00}
Chu~Min Li.
\newblock Integrating equivalency reasoning into {Davis-Putnam} procedure.
\newblock In Henry~A. Kautz and Bruce~W. Porter, editors, {\em Proceedings of
  the Seventeenth National Conference on Artificial Intelligence and Twelfth
  Conference on on Innovative Applications of Artificial Intelligence, July 30
  - August 3, 2000, Austin, Texas, {USA}}, pages 291--296. {AAAI} Press / The
  {MIT} Press, 2000.
\newblock URL: \url{http://www.aaai.org/Library/AAAI/2000/aaai00-045.php}.

\bibitem{Li03}
Chu~Min Li.
\newblock Equivalent literal propagation in the {DLL} procedure.
\newblock {\em Discret. Appl. Math.}, 130(2):251--276, 2003.
\newblock URL: \url{https://doi.org/10.1016/S0166-218X(02)00407-9}, \href
  {http://dx.doi.org/10.1016/S0166-218X(02)00407-9}
  {\path{doi:10.1016/S0166-218X(02)00407-9}}.

\bibitem{Soos10}
Mate Soos.
\newblock Enhanced gaussian elimination in {DPLL}-based {SAT} solvers.
\newblock In Daniel~Le Berre, editor, {\em {POS-10.} Pragmatics of SAT,
  Edinburgh, UK, July 10, 2010}, volume~8 of {\em EPiC Series in Computing},
  pages 2--14. EasyChair, 2010.
\newblock URL: \url{https://easychair.org/publications/paper/j1D}.

\bibitem{SoosM19}
Mate Soos and Kuldeep~S. Meel.
\newblock {BIRD:} engineering an efficient {CNF-XOR} {SAT} solver and its
  applications to approximate model counting.
\newblock In {\em The Thirty-Third {AAAI} Conference on Artificial
  Intelligence, {AAAI} 2019, The Thirty-First Innovative Applications of
  Artificial Intelligence Conference, {IAAI} 2019, The Ninth {AAAI} Symposium
  on Educational Advances in Artificial Intelligence, {EAAI} 2019, Honolulu,
  Hawaii, USA, January 27 - February 1, 2019}, pages 1592--1599. {AAAI} Press,
  2019.
\newblock URL: \url{https://doi.org/10.1609/aaai.v33i01.33011592}, \href
  {http://dx.doi.org/10.1609/aaai.v33i01.33011592}
  {\path{doi:10.1609/aaai.v33i01.33011592}}.

\bibitem{SoosNC09}
Mate Soos, Karsten Nohl, and Claude Castelluccia.
\newblock Extending {SAT} solvers to cryptographic problems.
\newblock In Oliver Kullmann, editor, {\em Theory and Applications of
  Satisfiability Testing - {SAT} 2009, 12th International Conference, {SAT}
  2009, Swansea, UK, June 30 - July 3, 2009. Proceedings}, volume 5584 of {\em
  Lecture Notes in Computer Science}, pages 244--257. Springer, 2009.
\newblock URL: \url{https://doi.org/10.1007/978-3-642-02777-2\_24}, \href
  {http://dx.doi.org/10.1007/978-3-642-02777-2\_24}
  {\path{doi:10.1007/978-3-642-02777-2\_24}}.

\bibitem{gadgets}
Luca Trevisan, Gregory~B. Sorkin, Madhu Sudan, and David~P. Williamson.
\newblock Gadgets, approximation, and linear programming.
\newblock {\em {SIAM} J. Comput.}, 29(6):2074--2097, 2000.

\end{thebibliography}

\end{document}